\documentclass{article}

\usepackage{arxiv}

\usepackage[utf8]{inputenc} 
\usepackage[T1]{fontenc}    
\usepackage{hyperref}       
\usepackage{url}            
\usepackage{booktabs}       
\usepackage{amsfonts}       
\usepackage{nicefrac}       
\usepackage{microtype}      
\usepackage{lipsum}		
\usepackage{times}
\usepackage{epsfig}
\usepackage{graphicx}
\usepackage{amsmath}
\usepackage{amssymb}
\usepackage{cite}
\usepackage{cuted}
\usepackage{flushend}
\usepackage{algorithm}
\usepackage{algorithmic}
\usepackage{subcaption}
\usepackage{amsthm}
\usepackage{mathrsfs}
\usepackage{tabularx,booktabs}
\usepackage{grffile}

\usepackage{makecell}
\usepackage{multirow}
\usepackage{hhline}
\usepackage{cellspace}
\usepackage{mathtools}
\DeclarePairedDelimiterX\set[1]\{\}{\nonscript\,#1\nonscript\,}

\newtheoremstyle{mystyle}
{}
{}
{\itshape}
{}
{\bfseries}
{.}
{ }
{}

\theoremstyle{mystyle}
\newtheorem{theorem}{Theorem}
\newcounter{subtheoremcounter}
\newtheorem{subtheorem}{Theorem}[subtheoremcounter]

\newtheorem{fact}{Fact}
\newcommand{\bs}{\boldsymbol}
\newcommand{\expectation}{\mathop{\mathbb{E}}}
\newcommand{\indicator}{\mathop{\mathbb{I}}}
\newcommand{\realnumber}{\mathop{\mathbb{R}}}
\newcommand{\tbf}{\textbf}

\newcommand{\etal}{\textit{et al}. }
\newcommand{\ie}{\textit{i}.\textit{e}., }
\newcommand{\eg}{\textit{e}.\textit{g}. }

\title{An Internal Covariate Shift Bounding Algorithm for Deep Neural Networks by Unitizing Layers' Outputs}

\author{
	You Huang\\
	Fuzhou University\\
	\texttt{youhuang0607@gmail.com} \\
	\And
	Yuanlong Yu\\
	Fuzhou University\\
	\texttt{yu.yuanlong@fzu.edu.cn} \\
}

\date{}
\renewcommand{\headeright}{}

\begin{document}
\maketitle

\begin{abstract}
	Batch Normalization (BN) techniques have been proposed to reduce the so-called Internal Covariate Shift (ICS) by attempting to keep the distributions of layer outputs unchanged. Experiments have shown their effectiveness on training deep neural networks. However, since only the first two moments are controlled in these BN techniques, it seems that a weak constraint is imposed on layer distributions and furthermore whether such constraint can reduce ICS is unknown. Thus this paper proposes a measure for ICS by using the Earth Mover (EM) distance and then derives the upper and lower bounds for the measure to provide a theoretical analysis of BN. The upper bound has shown that BN techniques can control ICS only for the outputs with low dimensions and small noise whereas their control is NOT effective in other cases. This paper also proves that such control is just a bounding of ICS rather than a reduction of ICS. Meanwhile, the analysis shows that the high-order moments and noise, which BN cannot control, have great impact on the lower bound. Based on such analysis, this paper furthermore proposes an algorithm that unitizes the outputs with an adjustable parameter to further bound ICS in order to cope with the problems of BN. The upper bound for the proposed unitization is noise-free and only dominated by the parameter. Thus, the parameter can be trained to tune the bound and further to control ICS. Besides, the unitization is embedded into the framework of BN to reduce the information loss. The experiments show that this proposed algorithm outperforms existing BN techniques on CIFAR-10, CIFAR-100 and ImageNet datasets.
\end{abstract}

\section{Introduction}
\label{sec:introduction}
Deep neural networks (DNNs) show good performance in image recognition~\cite{krizhevsky2012imagenet}, speech recognition~\cite{hinton2012deep} and other fields~\cite{tompson2014joint, mnih2013playing} in recent years. However, how to train DNNs is still a fundamental problem which is more complicated than training shallow networks because of deep architectures. It was commonly thought that stacking more layers suffers from the problem of vanishing or exploding gradients ~\cite{glorot2010understanding}, but there are some problems of training with unclear definitions. A problem called \emph{Internal Covariate Shift} (ICS)~\cite{ioffe2015batch} may hinder the convergence of training DNNs.

ICS is derived from \emph{Covariate Shift} (CS) that is caused by using data from two different distributions to respectively train and test a model, generally in the supervised learning process~\cite{shimodaira2000improving}. However, ICS only exists in the feed-forward networks. Considering the $l$th layer of a network with $L$ layers, a stack of following $L - l$ layers forms a local network $f_{l+1:L}$, whose input is the output of the $l$th layer. Thus, the distribution of the input is affected by all the former $l$ layers' weights. In detail, the objective function of $f_{l+1:L}$ is defined as
\begin{equation}
\label{eq:objectivefunction}
\mathcal{L}(\bs{\Theta}_{l+1:L}; {p_{l}^{(t)}}, p_{\bs{y}}) = {\expectation}_{\bs{x} \sim p_l^{(t)}, \bs{y} \sim p_{\bs{y}}(\cdot|\bs{x})} [h(\bs{x}, \bs{y}; \bs{\Theta}_{l+1:L})]
\end{equation}
where $p_{l}^{(t)}$ is the distribution of the $l$th layer's outputs at the $t$th iteration; $p_{\bs{y}}$ is the conditional distribution of the ground truth at the last layer given $\bs{x}$; $\bs{\Theta}_{l+1:L}$ is the weights of $f_{l+1:L}$; $h(\bs{x}, \bs{y}; \bs{\Theta}_{l+1:L})$ is the loss for a sample pair $(\bs{x}, \bs{y})$. We use Back Propagation algorithm to train networks. However, the objective function $\mathcal{L}(\bs{\Theta}_{l+1:L}; {p_{l}^{(t+1)}}, p_{\bs{y}})$ at the $t+1$th iteration would be different from the previous one as the distribution changes from $p_{l}^{(t)}$ to $p_{l}^{(t+1)}$. So using the gradients obtained from $\mathcal{L}(\bs{\Theta}_{l+1:L}; {p_{l}^{(t)}}, p_{\bs{y}})$ to update $\bs{\Theta}_{l+1:L}$ might not reduce $\mathcal{L}(\bs{\Theta}_{l+1:L}; {p_{l}^{(t+1)}}, p_{\bs{y}})$ due to the divergence between $p_{l}^{(t)}$ and $p_{l}^{(t+1)}$. Furthermore, the divergence will become much larger with increase of the number of network layers.

The technique called Batch Normalization (BN)~\cite{ioffe2015batch} has been proposed to reduce ICS by attempting to make the distributions remain unchanged. In practice, BN normalizes the outputs to control the first two moments, \ie, mean and variance, and uses two adjustable parameters to recover the information lost in normalizing outputs. Empirical results have shown that BN can speed up network training and also improve the success rate~\cite{he2016deep, silver2017mastering}. However, whether BN can really reduce ICS is not clear in theory. It is obviously that the first issue is about how to measure the divergence. Furthermore, since BN techniques only control the first and second moments, the constraint imposed on distributions by BN is weak. So how to theoretically analyze the bounding of ICS imposed by BN techniques is the second issue.

Meanwhile, some experiments have shown that the performance gain of BN seems unrelated to the reduction of ICS~\cite{santurkar2018does}. In fact, ICS always exists when we train networks based on the fact that the gradient strategy must give the weight update in order to train the network such that the distribution of each layer varies. Furthermore, in the case of gradient vanishing, the ICS is totally eliminated. However, the network training cannot work. This case illustrates that very slight ICS cannot support effective training. Thus, it seems that controlling ICS instead of eliminating it is effective for training networks. So how to control ICS in order to improve network training is another challenge issue.

This paper proposed an ICS measure, \ie, the divergence, by using the Earth Mover (EM) distance~\cite{villani2008optimal}, inspired by the success of Wasserstein Generative Adversarial Networks (WGAN)~\cite{arjovsky2017wasserstein}. Furthermore, this paper simplifies the measure by leveraging the Kantorovich-Rubinstein duality~\cite{villani2008optimal}.

Based on this proposed ICS measure, this paper furthermore derives an upper bound of ICS between $p_l^{(t+\Delta t)}$ and $p_l^{(t)}$ at the $t+\Delta t$th and $t$th iterations, respectively. The upper bound has shown that BN techniques can control the bounding of ICS in the low-dimensional case with small noise. Otherwise, the upper bound is out of control by using BN techniques. So it is required to analyze the lower bound of ICS especially for non-trivial distributions. Thus, this paper also derives a lower bound of ICS. The result has shown that the high-order moments and noise have great impact on the lower bound.

In order to control ICS, this paper proposes an algorithm that unitizes the normalized outputs. It is obvious that normalizing the outputs can introduce the moment-dependent upper bound, but such normalization would degrade when the moment estimation is not accurate. In contrast, unitizing the outputs in this proposed algorithm can lead to a constant upper bound without noise. However, by simply unitizing the outputs, the bound is very tight such that the weights cannot be updated in a reasonable range. Instead, this paper introduces a trainable parameter $\alpha$ in the unitization, such that the upper bound is adjustable and ICS can be further controlled by fine-tuning $\alpha$. It is important that the proposed unitization is embedded into the BN framework in order to reduce the information loss. The experiments show that the proposed unitization can outperform existing BN techniques in the benchmark datasets including CIFAR-10, CIFAR-100~\cite{krizhevsky2009learning} and ImageNet~\cite{russakovsky2015imagenet}.

\section{Related work}
\label{sec:related work}
Batch Normalization aims to reduce ICS by stabilizing the distributions of layers' outputs~\cite{ioffe2015batch}. In fact, BN only controls the first two moments by normalizing the outputs, which is inspired by the idea of whitening the outputs to make the training faster~\cite{lecun2012efficient}. However, BN is required to work with a sufficiently large batch size in order to reduce the noise of moments, and will degrade when the restriction on batch sizes is more demanding in some tasks~\cite{girshick2015fast, ren2015faster, he2017mask}. Thus the methods including LN~\cite{ba2016layer}, IN~\cite{dmitry2016instance} and GN~\cite{wu2018group} have been proposed. These variants estimate the moments within each sample, mitigating the impact of micro-batch. Besides, a method called Kalman Normalization (KN) addresses this problem by the merits of Kalman Filtering~\cite{wang2018kalman}.

Other methods inspired by BN have been proposed to improve network training. Weight Normalization decouples the length of the weight from the direction by re-parameterizing the weights, and speeds up convergence of the training\cite{salimans2016weight}. Cho and Lee regard the weight space in a BN layer as a Riemannian manifold and provide a new learning rule following the intrinsic geometry of this manifold\cite{cho2017riemannian}. Cosine Normalization uses cosine similarity and bounds the results of dot product, addressing the problem of the large variance~\cite{luo2017cosine}. Wu \etal propose the algorithm normalizing the layers' inputs with $l_1$ norm to reduce computation and memory~\cite{wu2018l1}. Huang \etal propose Decorrelated Batch Normalization that whitens instead of normalizing the activations~\cite{huang2018decorrelated}.

However, there is no complete analysis in theory for BN. Santurkar \etal attempt to demonstrate that the performance gain of BN is unrelated to the reduction of ICS by experiments~\cite{santurkar2018does}. However, the first experiment only shows that BN might improve network training by multiple ways besides the reduction of ICS. In the second experiment, the difference between gradients is an unsuitable ICS measure since the gradients are sensitive and the accurate estimations require sufficient samples. In addition, the theoretical analysis provided by Kohler \etal~\cite{kohler2018exponential} require strong assumptions without considering the deep layers. The reason why BN works still remains unclear.

\section{Unitization}
\label{sec:unitization}
The EM distance requires weak assumptions, and has been empirically proved to be effective in improving Generative Adversarial Networks (GANs)~\cite{goodfellow2014generative}, which replaces the traditional KL-divergence in making the objective function~\cite{arjovsky2017wasserstein}. According to the EM distance, the ICS measure for the $l$th layer's outputs is defined as
\begin{equation}
\label{eq:emd}
W(p_l^{(t+\Delta t)}, p_l^{(t)}) = \inf_{\gamma \in \prod(p_l^{(t+\Delta t)}, p_l^{(t)})} {\expectation}_{(\bs{x}, \bs{y}) \sim \gamma}\big[||\bs{x}-\bs{y}||\big]
\end{equation}
where $\prod(p_l^{(t+\Delta t)}, p_l^{(t)})$ denotes the set of all joint distributions whose marginals are $p_l^{(t+\Delta t)}$ and $p_l^{(t)}$, respectively~\cite{arjovsky2017wasserstein}. Then, by the Kantorovich-Rubinstein duality~\cite{villani2008optimal}, the EM distance Eq.(\ref{eq:emd}) can be rewritten as
\begin{equation}
\label{eq:ics}
W(p_l^{(t+\Delta t)}, p_l^{(t)}) = \sup_{||f||_L \le 1} {\expectation}_{\bs{x}\sim p_l^{(t+\Delta t)}}[f(\bs{x})] - {\expectation}_{\bs{y} \sim p_l^{(t)}}[f(\bs{y})]
\end{equation}
where the distance is obtained by optimizing $f$ over the 1-Lipschitz function space (see the algorithm that estimates the EM distance in the appendix).

\subsection{The Upper Bound}
For $d$-dimensional outputs of the $l$th layer, denote by $\bs{\mu}^{(t)} = (\mu_1^{(t)}, \mu_2^{(t)}, \ldots, \mu_d^{(t)})$ and $(\bs{\sigma}^{(t)})^2 = ((\sigma_1^{(t)})^2, (\sigma_2^{(t)})^2, \ldots, (\sigma_d^{(t)})^2)$ the mean and variance of the distribution $p_l^{(t)}$, respectively. The upper bound over $W(p_l^{(t+\Delta t)}, p_l^{(t)})$ is formed by the first two moments (see the proofs of all the theorems in the appendix).
\begin{theorem}
	Suppose that $|\mu_i^{(t)}| < \infty, |\mu_i^{(t + \Delta t)}| < \infty, 1 \le i \le d$. Then,
	\begin{equation*}
	\label{eq:bnupperbound}
	\begin{split}
	W(p_l^{(t+\Delta t)}, p_l^{(t)}) \le & \sum_{i=1}^d (\sigma_{i}^{(t+\Delta t)})^2 + \sum_{i=1}^d (\sigma_{i}^{(t)})^2 + \Big(\sum_{i=1}^d (\mu_i^{(t+\Delta t)} - \mu_i^{(t)})^2\Big)^{\frac{1}{2}} + 2
	\end{split}
	\end{equation*}
\end{theorem}
In BN, the output is normalized by the estimated mean $\hat{\mu_i}$ and standard deviation $\hat{\sigma_i}$. Thus, for the normalized output, assume that $\mu_i^{(t)} = \epsilon_{\mu, i}^{(t)}, (\sigma_i^{(t)})^2 = 1 + \epsilon_{\sigma^2, i}^{(t)}, 1 \le i \le d$, where $\epsilon_{\mu, i}^{(t)}, \epsilon_{\sigma^2, i}^{(t)}, 1 \le i \le d$ are noise. According to the above theorem, the upper bound is
\begin{equation}
\label{eq:bnnoise}
\begin{split}
W(p_l^{(t+\Delta t)}, p_l^{(t)}) \le & \sum_{i=1}^d \epsilon_{\sigma^2, i}^{(t+\Delta t)} + \sum_{i=1}^d \epsilon_{\sigma^2, i}^{(t)} + 2d + \Big(\sum_{i=1}^d (\epsilon_{\mu, i}^{(t+\Delta t)} - \epsilon_{\mu, i}^{(t)})^2\Big)^{\frac{1}{2}} + 2
\end{split}
\end{equation}
It's obvious that normalizing the outputs by noise-free moments will lead to a constant upper bound and impose a constraint on ICS. In contrast, the distance for the unnormalized outputs is unbounded (see an example of the unbounded distance in the appendix). Nevertheless, the noise cannot be controlled in practice, and for high-dimensional outputs, the bound in Eq.(\ref{eq:bnnoise}) might be too loose to constraint the distance due to a large $d$. In this case, the ICS for the non-trivial distributions cannot be bounded effectively by controlling the first two moments as BN techniques have done. Then, the analysis of the lower bound is required.

\subsection{The Lower Bound}
For convenience, let $\bs{x} = (x_1, x_2, \ldots, x_d)$ and $\bs{y} = (y_1, y_2, \ldots, y_d)$. Then, the lower bound on the distance is obtained by constructing a $1$-Lipschitz function.
\begin{theorem}
	Suppose that $C > 0$ is a real number, and $p \ge 2$ is an integer. Then,
	\begin{equation*}
	\label{eq:lowerbound}
	\begin{split}
	W(p_l^{(t + \Delta t)}, p_l^{(t)}) = \sup_{||f||_L \le 1} {\expectation}_{\bs{x} \sim p_l^{(t + \Delta t)}}[f(\bs{y})] - {\expectation}_{\bs{y} \sim p_l^{(t)}}[f(\bs{y})] \ge \big|{\expectation}_{\bs{x} \sim p_l^{(t + \Delta t)}}[f_{p, C}(\bs{x})] - {\expectation}_{\bs{y} \sim p_l^{(t)}}[f_{p, C}(\bs{y})]\big| \\
	\end{split}
	\end{equation*}
	where $f_{p, C}$ is the $1$-Lipschitz function defined as
	\begin{equation*}
	\begin{split}
	f_{p, C}(\bs{x}) = \dfrac{1}{pC^{p-1}d^{\frac{1}{2}}}\bigg(\sum_{|x_i| \le C} x_i^p + \sum_{x_i < -C} (-C)^p + \sum_{x_i > C} C^p\bigg)
	\end{split}
	\end{equation*}
\end{theorem}
To simplify the analysis, assume that the supports of the distributions are subsets of $[-C_0, C_0]^d$ for some $C_0 > 0$. Then, the lower bound is formed by the $p$th-order moments. For $p > 2$, it's straightforward that the high-order moments affect the lower bound, which cannot be controlled by BN especially in the case of the relaxed upper bound. On the other hand, for $p = 2$ and the normalized output, the lower bound is
\begin{equation}
\label{eq:bnlowerbound}
\begin{split}
W(p_l^{(t + \Delta t)}, p_l^{(t)}) \ge \dfrac{1}{2C_0d^{\frac{1}{2}}}\bigg|\sum_{i=1}^d (\epsilon_{\mu, i}^{(t+\Delta t)})^2 + \epsilon_{\sigma^2, i}^{(t+\Delta t)} - (\epsilon_{\mu, i}^{(t)})^2 - \epsilon_{\sigma^2, i}^{(t)} \bigg|
\end{split}
\end{equation}
The lower bound in Eq.(\ref{eq:bnlowerbound}) is dominated by the noise. Thus, BN degrades in such case especially for micro-batches. Some methods have been proposed, \eg, GN, to reduce the noise of the moments rather than eliminating the noise. So the lower bound is still dependent on moments.

Based on such analysis of BN, this paper proposes an algorithm with an adjustable upper bound that is noise-free and moment-independent to further bound the distance.

\subsection{Vanilla Unitization}
The proposed algorithm unitizes layers' outputs, and the vanilla unitization transformation is defined as
\begin{equation}
g(\bs{x}) = \left\{
\begin{array}{ll}
\dfrac{\bs{x}}{||\bs{x}||_2} &, ||\bs{x}||_2 \neq 0 \\
\bs{c} &, ||\bs{x}||_2  = 0
\end{array}
\right.
s\end{equation}
where $\bs{c}$ is a constant unit vector. Similarly, the upper bound for the unitized output is given. In fact, the EM distance for $g(\bs{x})$ is defined as
\begin{equation}
\begin{split}
W(p_U^{(t + \Delta t)}, p_U^{(t)}) =\sup_{||f||_L \le 1} {\expectation}_{\bs{x}\sim p_l^{(t+\Delta t)}}[f(g(\bs{x}))] - {\expectation}_{\bs{y} \sim p_l^{(t)}}[f(g(\bs{y}))]
\end{split}
\end{equation}
where $p_U^{(t)}$ is the distribution of the unitized outputs.
\stepcounter{subtheoremcounter}
\stepcounter{subtheoremcounter}
\stepcounter{subtheoremcounter}
\begin{subtheorem}
	Suppose that for $\bs{x} \sim p_l^{(t)}$, $g(\bs{x}) \sim p_U^{(t)}$. Then,
	\begin{equation*}
	\begin{split}
	W(p_U^{(t + \Delta t)}, p_U^{(t)}) \le 2
	\end{split}
	\end{equation*}
\end{subtheorem}
For $g$, the upper bound is absolutely a constant independent of all parameters including $d$. Then, ICS for the unitized outputs is exactly bounded in spite of the distribution $p_l^{(t)}$. Hence, by unitizing the outputs, ICS is fully controlled by this constant bound in fact. However, the constant upper bound leads to the other problem. For $t = 0$ and any $\Delta t > 0$, the distribution $p_U^{(\Delta t)}$ is constrained such that the distance between $p_U^{(\Delta t)}$ and $p_U^{(0)}$ is no more than the constant $2$. This might be a severe problem, especially when the network is poorly initialized. Thus, the unitization has to be modified.

\subsection{Modified Unitization}
To alleviate the problem of the very tight bound, define the transformation, which partly unitizes the outputs, as
\begin{equation}
g(\bs{x}; \alpha) = \left\{
\begin{array}{ll}
\bs{c} &, ||\bs{x}||_2 = 0, \alpha = 1\\
\dfrac{\bs{x}}{\alpha ||\bs{x}||_2 + (1 - \alpha) \times 1} &, other\\
\end{array}
\right.
\label{eq:partlyunitized}
\end{equation}
where $\alpha \in [0, 1]$ is a parameter. Analogously, the upper bound w.r.t. $g(\bs{x}; \alpha)$ is given.
\begin{subtheorem}
	Suppose that for $\alpha \in [0, 1]$ and $\bs{x} \sim p_l^{(t)}$, $g(\bs{x}; \alpha) \sim p_U^{(t)}$. Then,
	\begin{equation*}
	\begin{split}
	W(p_U^{(t + \Delta t)}, p_U^{(t)})
	\le {\indicator}_{\alpha = 0}(\alpha) \cdot ({\expectation}_{\bs{x} \sim p_l^{(t+\Delta)}}[||\bs{x}||_2]
	+ {\expectation}_{\bs{y} \sim p_l^{(t)}}[||\bs{y}||_2]) + {\indicator}_{\alpha > 0}(\alpha) \cdot \dfrac{2}{\alpha}
	\end{split}
	\end{equation*}
\end{subtheorem}
Note that $\alpha = 0$ implies $g(\bs{x}; \alpha)$ is an identity mapping, and $\alpha > 0$ implies the distance is exactly bounded by $2 / \alpha$. Thus, the bound is dominated by $\alpha$, and the very bound is obtained by fine-tuning $\alpha$ over $[0, 1]$.

Furthermore, considering a set of parameters $\bs{\alpha} = (\alpha_1, \alpha_2, \ldots, \alpha_d) \in [0, 1]^d$, the general unitization is defined as
\begin{equation}
g(\bs{x}; \bs{\alpha}) =
\left\{
\begin{array}{ll}
\bs{0} &, ||\bs{x}||_2 = 0 \\
\big((||\bs{x}||_2 - 1) \cdot \text{diag}(\bs{\alpha}) + E\big)^{-1} \bs{x} &, ||\bs{x}||_2 > 0
\end{array}
\right.
\label{eq:multivariateunitized}
\end{equation}
where $\text{diag}(\bs{\alpha})$ is a diagonal matrix of $\bs{\alpha}$ and $E$ is a unit matrix. Likewise, the upper bound for $g(\bs{x}; \bs{\alpha})$ is given.
\begin{subtheorem}
	Suppose that for $\bs{\alpha} = (\alpha_1, \alpha_2, \ldots, \alpha_d)$, where $\alpha_i \in [0, 1], 1 \le i \le d$, and $\bs{x} \sim p_l^{(t)}$, $g(\bs{x}; \bs{\alpha}) \sim p_U^{(t)}$. Then,
	\begin{equation*}
	\begin{split}
	W(p_U^{(t + \Delta t)}, p_U^{(t)}) \le  {\indicator}_{\min_j \alpha_j > 0}(\bs{\alpha}) \cdot \dfrac{2}{\min_j \alpha_j} 
	+ {\indicator}_{\min_j \alpha_j = 0}(\bs{\alpha}) \cdot ({\expectation}_{\bs{x} \sim p_l^{(t+\Delta)}}[||\bs{x}||_2] 
	+ {\expectation}_{\bs{y} \sim p_l^{(t)}}[||\bs{y}||_2] + 2)
	\end{split}
	\end{equation*}
\end{subtheorem}
The minimum $\alpha_* = \min_j \alpha_j$ dominates the upper bound. If $\alpha_* = 0$, then there exists $i$ such that the scale of $x_i$ remains unchanged after the unitization, and the EM distance for the marginal distribution of $x_i$ is unbounded. In contrast, the constant bound $2 / \alpha^*$ is obtained by $\alpha_* > 0$. Furthermore, if $\alpha_i$ is fixed for some $i$, the other parameters $\alpha_j, j \neq i$ can be freely fine-tuned over $[\alpha_i, 1]$ without changing the bound. Thus, $g(\bs{x}; \bs{\alpha})$ is more flexible, and used in the proposed algorithm. However, the unitized outputs lose some information, \eg, similarity between the samples, which cannot be recovered by an affine transformation like that in BN. The smaller bound leads to more information loss, and a trade-off is required. Then, the unitization algorithm is given.

\subsection{Algorithm}
For a network, $\bs{\alpha}$ in each unitization layer is trained with the weights to reduce the objective function. However, since $\bs{\alpha} \in [0, 1]^d$, the training would lead to a constrained optimization problem. To avoid the problem and make the training stable, this paper uses a simple interpolation method for Eq.(\ref{eq:multivariateunitized}). The practical transformation is defined as
\begin{equation}
\label{eq:implementedunitization}
g(\bs{x}; \bs{\alpha}) = \bigg[\frac{1}{\sqrt{||\bs{x}||_2^2+\epsilon}} \bs{\alpha} + (\bs{1} - \bs{\alpha})\bigg] \odot \bs{x}
\end{equation}
where $\epsilon > 0$ makes the non-zero denominator, and $\odot$ represents element-wise production. In the proposed algorithm, the unitization Eq.(\ref{eq:implementedunitization}) is embedded into the framework of BN to reduce the information loss. In fact, only the unitization might require large $\bs{\alpha}$ to bound the EM distance with more information loss. In contrast, there is less information loss in BN since  similarity between the normalized outputs remains unchanged and the affine transformation in BN can recover some information. Thus, the proposed algorithm integrates these two techniques to bound ICS with reasonable information loss. The algorithm is presented in Algorithm~\ref{alg:unitization}, where element-wise division is also denoted by $/$. The moments $\bs{\mu}$ and $\bs{\sigma}^2$ in inference is computed in the same way~\cite{ioffe2015batch}.

\begin{algorithm}[htb]
	\textbf{Input}: dataset $\{\bs{x}_{i}\}_{i=1}^{n}$, trainable parameters $\bs{\alpha}, \bs{\gamma}$ and $\bs{\beta}$
	
	\textbf{Output}: unitized results $\{\bs{y}_{i}\}_{i=1}^{n}$
	
	\begin{algorithmic}[1]
		\STATE{$\bs{\mu} \gets \frac{1}{n}\sum_i \bs{x}_i$}
		\STATE{$\bs{\sigma}^2 \gets \frac{1}{n}\sum_i (\bs{x}_i-\bs{\mu})^2$}
		\FOR{$i \gets 1$ to $n$}
		\STATE{$\hat{\bs{x}}_i \gets (\bs{x}_i-\bs{\mu})/\sqrt{\bs{\sigma}^2 + \epsilon}$}
		\STATE{$p \gets 1 / \sqrt{||\hat{\bs{x}}_i||_2^2 + \epsilon} $}
		\STATE{$\overline{\bs{x}}_i \gets [p \bs{\alpha} + (\bs{1}-\bs{\alpha}) ]\odot \hat{\bs{x}}_i$}
		\STATE{$\bs{y}_{i} \gets \bs{\gamma} \odot \overline{\bs{x}}_i + \bs{\beta}$}
		\ENDFOR
	\end{algorithmic}
	\caption{Unitization Algorithm}
	\label{alg:unitization}
\end{algorithm}

\begin{figure*}[t!]
	\centering
	\begin{subfigure}[t]{0.25\textwidth}
		\centering
		\includegraphics[width=40mm,scale=0.5]{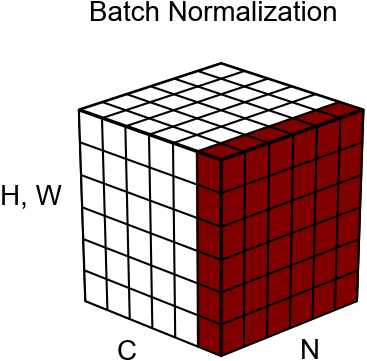}
		\caption{}
		\label{fig:unitizedcnns:a}
	\end{subfigure}
	\begin{subfigure}[t]{0.25\textwidth}
		\centering
		\includegraphics[width=40mm,scale=0.5]{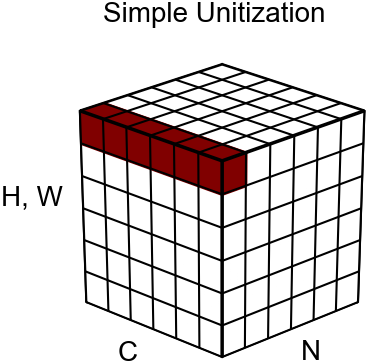}
		\caption{}
		\label{fig:unitizedcnns:b}
	\end{subfigure}
	\begin{subfigure}[t]{0.25\textwidth}
		\centering
		\includegraphics[width=40mm,scale=0.5]{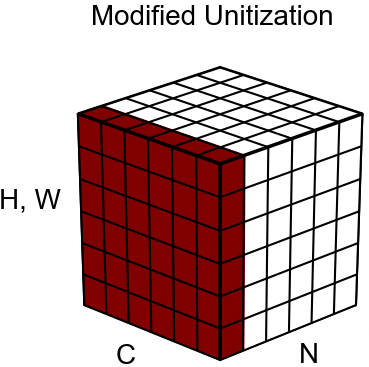}
		\caption{}
		\label{fig:unitizedcnns:c}
	\end{subfigure}
	\caption{Different methods for estimating the statistics. Like the visualization of normalization methods~\cite{wu2018group}, each subplot shows a feature map tensor, with $N$, $C$ and $(H, W)$ as the batch axis, the channel axis and the spatial axes, respectively. (a) shows the values of the pixels in red are used to compute the moments $\bs{\mu}$ and $\bs{\sigma}$ in BN, while (b) and (c) show the values of the pixels are used to obtain the norm. The estimated moments and norms are shared within these pixels.}
	\label{fig:unitizedcnns}
\end{figure*}

\subsection{Unitized Convolutional Layers}

\begin{algorithm}[htb]
	\textbf{Input:} dataset $\mathcal{D} = \big\{x_{ij}^{(k)}|1\le k \le N, 1\le i \le C, 1 \le j$ $\le HW \big\}$, trainable parameters $\bs{\alpha}, \bs{\gamma}$ and $\bs{\beta}$
	
	\textbf{Output:} unitized results $\big\{y_{ij}^{(k)}|1\le k \le N, 1\le i \le C, 1 \le j \le HW \big\}$
	
	\begin{algorithmic}[1]
		\FOR{$k \gets 1$ to $N$}
		\STATE{$s \gets 0$}
		\FOR{$i \gets 1$ to $C$}
		\FOR{$j \gets 1$ to $HW$}
		\STATE{$\hat{x}_{ij}^{(k)} = \text{BN}(x_{ij}^{(k)};\mathcal{D})$}
		\STATE{$s \gets s + \hat{x}_{ij}^{(k)2}$}
		\ENDFOR
		\ENDFOR
		\STATE{$s \gets s / (nHW)$}
		\STATE{$p \gets 1/\sqrt{s + \epsilon}$}
		
		\FOR{$i \gets 1$ to $C$}
		\FOR{$j \gets 1$ to $HW$}
		\STATE{$\bar{x}_{ij}^{(k)} \gets [p\alpha_i + (1-\alpha_i)]\hat{x}_{ij}^{(k)}$}
		\STATE{$y_{ij}^{(k)} \gets \gamma_i \bar{x}_{ij}^{(k)} + \beta_i$}
		\ENDFOR
		\ENDFOR
		\ENDFOR
	\end{algorithmic}
	\caption{Unitization Algorithm for Convolutional Layers}
	\label{alg:unitizedcnn}
\end{algorithm}

To take into account the spatial context of image data, this paper also proposes the unitized convolutional layers. As recommended by~\cite{ioffe2015batch}, the moments $\bs{\mu}$ and $\bs{\sigma}^2$ in Algorithm~\ref{alg:unitization} are computed over the whole mini-batch at different locations w.r.t. a feature map, and they are shared within the same feature map (Figure~\ref{fig:unitizedcnns}(a)). But the norm in the unitization is computed in a different way. How to compute the norm is determined by the definition of a single sample for image data. A simple algorithm follows the idea of BN in convolutional layers, where the pixels at the same location in all channels is regarded as a single sample (Figure~\ref{fig:unitizedcnns}(b)), and then this sample will be unitized by its norm. However, this algorithm scales the pixels by location-dependent norms, ignoring the spatial context.

Thus, the computation of the norms has to be modified. Instead, the pixels at all locations in all feature maps within an image form a single sample (Figure~\ref{fig:unitizedcnns}(c)). The sample's norm will be location-independent and shared within these pixels. However, this will lead to a very large norm for a pixel. As the inverse of the norm, $p$ in Algorithm~\ref{alg:unitization} will be relatively small and make $p \bs{\alpha}$ be ignored when $\bs{\alpha}$ is being fine-tuned. Then $\bs{\alpha}$ only scales $\hat{\bs{x}}_i$ by $1 - \bs{\alpha}$, but the scale has been controlled by $\bs{\gamma}$. Hence, the norm is divided by a constant related to the number of the pixels before the unitization.

The modified unitization is presented in Algorithm~\ref{alg:unitizedcnn}, where $x_{ij}^{(k)}$ denotes the value at the $j$th location in the $i$th feature map, generated from the $k$th training sample; $\hat{x}_{ij}^{(k)}$, $\bar{x}_{ij}^{(k)}$ and $y_{ij}^{(k)}$ are defined in the same way; $\text{BN}(\cdot;\mathcal{D})$ denotes the normalization transformation for convolutional layers~\cite{ioffe2015batch} using the dataset $\mathcal{D}$ without the affine transformation; $\alpha_i$, $\gamma_i$ and $\beta_i$ denote the $i$th element of $\bs{\alpha}$, $\bs{\gamma}$ and $\bs{\beta}$, respectively; the $n$ in Line 9 is a hyper-parameter that is set to $HW$ by default.

\begin{figure*}[t!]
	\centering
	\begin{subfigure}[t]{\dimexpr0.20\textwidth+15pt\relax}
		\makebox[15pt]{\raisebox{30pt}{\rotatebox[origin=c]{90}{BN}}}%
		\includegraphics[width=\dimexpr\linewidth-15pt\relax]
		{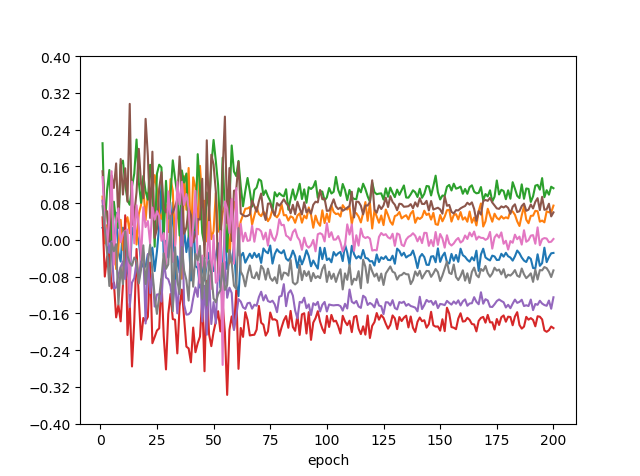}
		\makebox[15pt]{\raisebox{30pt}{\rotatebox[origin=c]{90}{Unitization}}}%
		\includegraphics[width=\dimexpr\linewidth-15pt\relax]
		{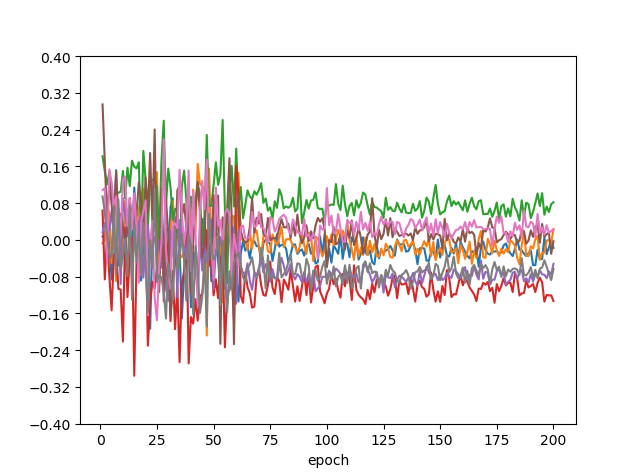}
		\caption{Mean}
	\end{subfigure}\hfill
	\begin{subfigure}[t]{0.20\textwidth}
		\includegraphics[width=\textwidth]
		{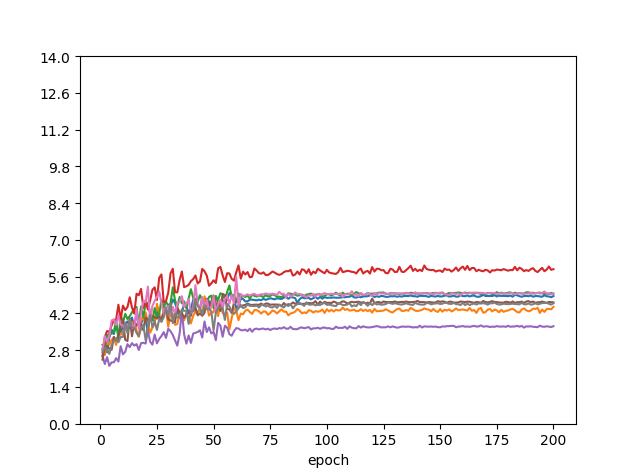}
		\includegraphics[width=\textwidth]
		{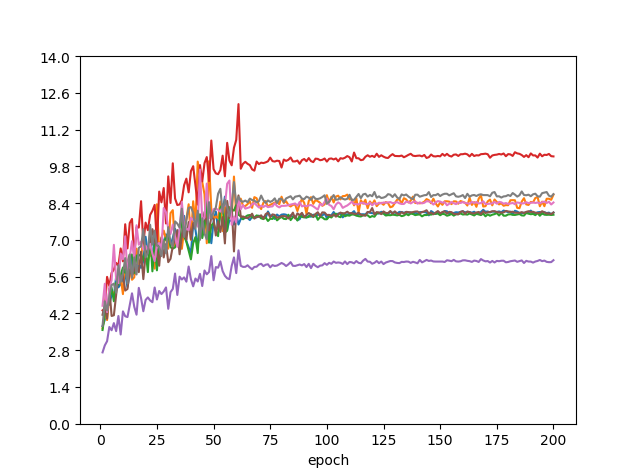}
		\caption{Variance}
	\end{subfigure}\hfill
	\begin{subfigure}[t]{0.20\textwidth}
		\includegraphics[width=\textwidth]
		{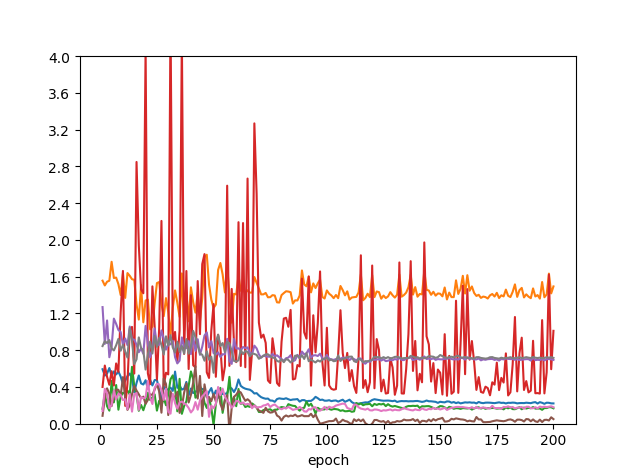}
		\includegraphics[width=\textwidth]
		{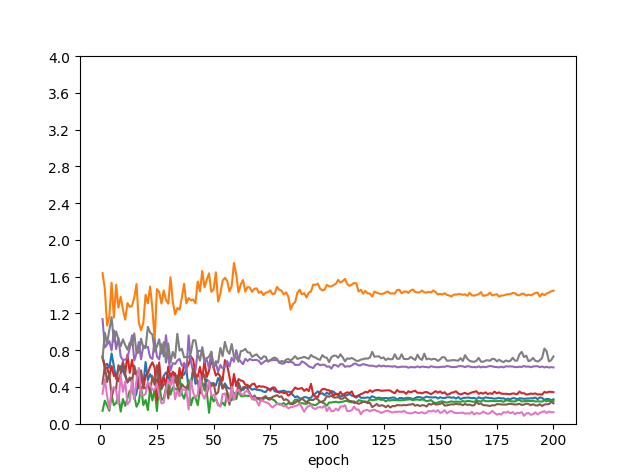}
		\caption{Skewness}
	\end{subfigure}\hfill
	\begin{subfigure}[t]{0.20\textwidth}
		\includegraphics[width=\textwidth]
		{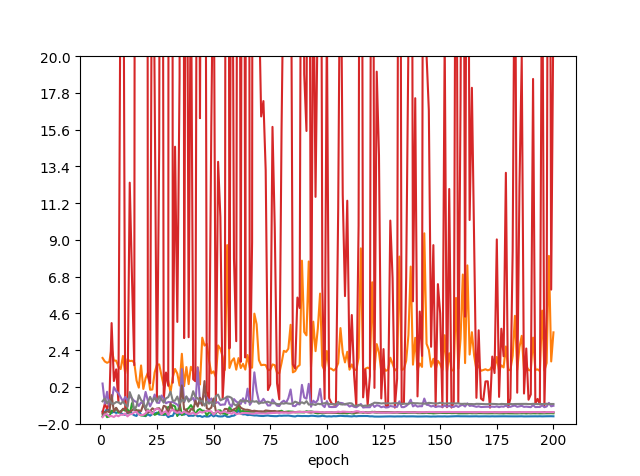}
		\includegraphics[width=\textwidth]
		{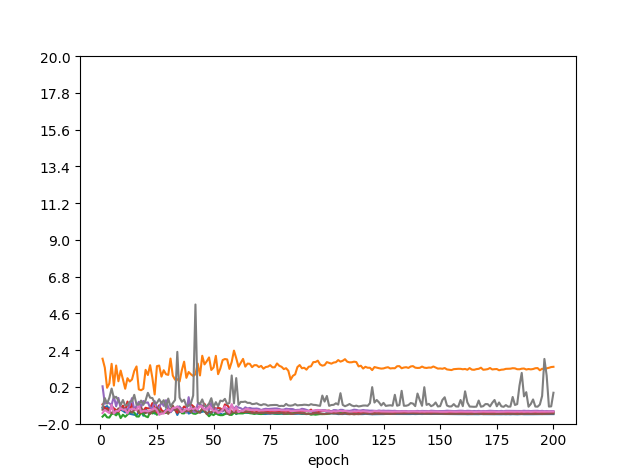}
		\caption{Kurtosis}
	\end{subfigure}
	\caption{Moments estimated over the 8-unit layers' outputs. Each line in a subfigure represents the moment w.r.t. a unit's output. In general, more stable moments are obtained by using the unitization.}
	\label{fig:moments}
\end{figure*}

\section{Experiments}
\label{sec:experiments}

\subsection{Estimated Moments}
To verify the ability of the unitization controlling higher moments, we train simple neural networks on the MNIST dataset~\cite{lecun1998gradient} and estimate the moments of a certain layer's outputs.

\tbf{Network Architectures}: The inputs of the networks are $28 \times 28$ images, followed by a stack of fully-connected layers with ReLU activations, consisting of 10 100-unit layers and one 8-unit layer. A BN/unitization layer follows each fully-connected layer. The networks end with a fully-connected layer for 10 classes.

\tbf{Implementation Details}: The networks are trained using Mini-batch Gradient Descent for 200 epochs, and the batch size is 128. The learning rate starts from 0.05 and is divided by 5 at the $61$th, $121$th and $161$th epochs. At the end of each epoch, the training samples are fed into the networks to get the normalized/unitized 8-unit layer's outputs. Then, we estimate the mean, variance, skewness and kurtosis over the outputs.

\tbf{Results}: As is shown in Figure~\ref{fig:moments}, the estimated mean and variance for both BN and the unitization layers are stable through the training. However, the estimated skewness and kurtosis are unstable w.r.t. the BN layer, with large fluctuation in the red line. By the lower bound in Eq.(\ref{eq:lowerbound}), the EM distance will be large. In contrast, there are more stable skewness and kurtosis of the unitized outputs. The proposed unitization controls the high-order moments

\subsection{Classification Results on CIFAR}
For the tasks of image recognition, we run the experiments on CIFAR-10 and CIFAR-100 datasets~\cite{krizhevsky2009learning}, following the data augmentation recommended by~\cite{lee2015deeply} (the experiments of comparing the unitization with GN~\cite{wu2018group} are provided in the appendix). Code is available at \url{https://github.com/unknown9567/unitization.git}.

\tbf{Network Architectures}: The ResNet-20, ResNet-110 and ResNet-164~\cite{he2016deep, he2016identity} with BN or the unitization are trained to compare the performance. The ResNets follow the general architecture~\cite{he2016identity} with full pre-activation blocks.

\tbf{Implementation Details}: In each experiment, for BN and the unitization, the networks are initialized by the same weights that are generated using the method of~\cite{he2015delving} to reduce impacts of initialization. Each network is trained by Mini-batch Gradient Descent with Nesterov's momentum, and the same learning rate in the moment experiments is used. The momentum is $0.9$ and the weight decay is $0.0005$. The mini-batch sizes are 128 and 64 in training $\{$ResNet-20, ResNet-110$\}$ and ResNet-164, respectively. Each network is evaluated after 200 epochs and the median accuracy of 5 runs is reported. The parameter $\bs{\alpha}$ is initialized to $\bs{0}$. The parameters $\bs{\gamma}$ and $\bs{\beta}$ are initialized to $\bs{1}$ and $\bs{0}$, respectively, as suggested by~\cite{ioffe2015batch}.

\begin{table}[h]
	\centering
	\captionsetup{labelsep=newline,justification=centering}
	\caption{Classification accuracy on the CIFAR-10 testing dataset.}
	\label{tb:cifar10}
	\begin{tabular}{lcl}
		\toprule
		Network                &    Mini-batch size        &     Accuracy\\
		\midrule
		ResNet-20 (BN)               &    128               &  91.79$\%$ \\
		ResNet-20 (Unitization)     &    128               & \textbf{92.21$\%$} \\
		\midrule
		ResNet-110 (BN)~\cite{he2016identity}               &    128               &  93.63$\%$ \\
		ResNet-110 (BN)               &    128               &  93.99$\%$ \\
		ResNet-110 (Unitization)     &    128               & \textbf{94.12$\%$} \\
		\midrule
		ResNet-164 (BN)~\cite{he2016identity}               &    128               &  94.54$\%$ \\
		ResNet-164 (BN)               &    64               &  94.34$\%$ \\
		ResNet-164 (Unitization)     &     64               & \textbf{94.62$\%$} \\
		\bottomrule
	\end{tabular}
\end{table}
\tbf{Results}: Table~\ref{tb:cifar10} shows the results on the CIFAR-10 dataset, where the results in~\cite{he2016identity} are also presented for comparison. The proposed algorithm has better performance, compared with BN, which raises the classification accuracy for each ResNet. But there is less improvement in the accuracy of the deeper network, which might be explained by the less benefit from stacking more layers in deeper networks. Actually, the ResNet-1001~\cite{he2016identity} only achieves an accuracy of $95.08\%$ that is the limit of these ResNets with BN. The accuracy of ResNet-164 in~\cite{he2016identity} is only $0.54\%$ less than that of the ResNet-1001, but using the unitization raises the accuracy by $0.08\%$.

The results on the CIFAR-100 dataset are reported in Table~\ref{tb:cifar100}, where the unitization still outperforms BN, with increase of over $1\%$ in the accuracy of each experiment.

\begin{table}[h]
	\centering
	\captionsetup{labelsep=newline,justification=centering}
	\caption{Classification accuracy on the CIFAR-100 testing dataset.}
	\label{tb:cifar100}
	\begin{tabular}{lcl}
		\toprule
		Network                &    Mini-batch size        &     Accuracy\\
		\midrule
		ResNet-20 (BN)               &    128               &  66.43$\%$ \\
		ResNet-20 (Unitization)    &    128               & \textbf{67.49$\%$} \\
		\midrule
		ResNet-110 (BN)              &    128               &  72.27$\%$ \\
		ResNet-110 (Unitization)    &    128               & \textbf{73.31$\%$} \\
		\midrule
		ResNet-164 (BN)~\cite{he2016identity}               &    128               &  75.67$\%$ \\
		ResNet-164 (BN)              &    64               &  76.56$\%$ \\
		ResNet-164 (Unitization)    &     64               & \textbf{77.58$\%$} \\
		\bottomrule
	\end{tabular}
\end{table}

\subsection{Classification Results on ImageNet}
We further evaluate the unitization on the ImageNet 2012 classification dataset~\cite{russakovsky2015imagenet}. The networks are trained on the 1.28M training images, and evaluated on the 50k validation images.  Only the scale augmentation~\cite{he2016deep, simonyan2014very} is used.

\tbf{Network Architectures}: Only the ResNet-101 with BN or the unitization is trained to compare the performance. The network follows the architecture~\cite{he2016deep} but with full pre-activation blocks~\cite{he2016identity}. Besides, the outputs of each shortcut connection and the final block are not normalized or unitized.

\tbf{Implementation Details}: Like the experiments on CIFAR datasets, the weights are initialized by the method~\cite{he2015delving}, shared between the experiments and trained by the gradient descent with the same momentum, but the weight decay is 0.0001. The learning rate starts from 0.01 and is divided by 5 at the $31$th, $61$th and $91$th epochs. The batch size is 64 for a single GPU. After 120 epochs, the networks are evaluated on the validation data by two methods. The first method resizes the images with the sorter side in $\{224, 256, 384, 480, 640\}$ and averages the scores over 42 crops at all scales (2 central crops for 224-scale images, 10 standard crops for other resized images). The second method adopts the fully-convolutional form and averages the scores over the same multi-scale images~\cite{he2016deep}. Besides, the $n$ in Line 9 of Algorithm~\ref{alg:unitizedcnn} is fixed and set to the same value in the training for the multi-scale images.

\begin{table}[h]
	\centering
	\captionsetup{labelsep=newline,justification=centering}
	\caption{Classification accuracy on the ImageNet dataset.}
	\label{tb:imagenet}
	\begin{tabular}{lccc}
		\toprule
		Algorithm                &    Method/Mini-batch size             &     Top-1           & Top-5 \\
		\midrule
		BN                   &    multi-scale crops/64  &  78.12$\%$          & 93.45$\%$\\
		Unitization              &    multi-scale crops/64  & \textbf{78.33$\%$}  & 93.22$\%$\\
		\midrule
		BN                   &    fully-convolution/64  &  76.47$\%$          & 93.02$\%$\\
		Unitization              &    fully-convolution/64  & \textbf{77.84$\%$}  & \textbf{93.33$\%$}\\
		\midrule
		BN~\cite{he2016deep} &    fully-convolution/256 &  80.13$\%$          & 95.40$\%$\\
		\midrule
		\bottomrule
	\end{tabular}
\end{table}
\tbf{Results}: In the results, the unitization outperforms BN in general, with only the top-5 accuracy for the first method less than that of BN. However, there is a performance gap between the reproduced result and the accuracy in~\cite{he2016deep}, which might be explained by the different implementation details including the data augmentation, the architecture and hyper-parameters like the batch size and the learning rate. But for the evaluation method of fully-convolution recommended by~\cite{he2016deep}, the accuracy increases by over 1$\%$ using the unitization. In general, the unitization shows higher performance for classification tasks.

\section{Conclusion}
\label{sec:conclusion}
This paper proposes an ICS measure by using the EM distance, and provides a theoretical analysis of BN through the upper and lower bounds. The moment-dependent upper bound has shown that BN techniques can effectively control ICS only for the low-dimensional outputs with small noise in the moments, but would degrade in other cases. Meanwhile, the high-order moments and noise that are out of BN's control have great impact on the lower bound. Then, this paper proposes the unitization algorithm with the noise-free and moment-independent upper bound. By training the parameter in the unitization, the bound can be fine-tuned to further control ICS. The experiments demonstrate the proposed algorithm's control of high-order moments and performance on the benchmark datasets including CIFAR-10, CIFAR-100 and ImageNet.

\bibliographystyle{unsrt}  
\bibliography{references}

\newpage
\appendix
\setcounter{theorem}{0}
\setcounter{subtheoremcounter}{0}

\section{Proofs of the Theorems}
\subsection{Useful facts}
There are two facts that can be derived directly from the EM distance, and they'll be used in the proof of the theorems.
\begin{fact}
	\label{fact:1}
	\begin{equation*}
	\begin{split}
	W(p_l^{(t+\Delta t)}, p_l^{(t)}) = & \sup_{||f||_L \le 1} {\expectation}_{\bs{x} \sim p_l^{(t+\Delta t)}}[f(\bs{x})] - {\expectation}_{\bs{y} \sim p_l^{(t)}} [f(\bs{y})] \\
	= & \sup_{||f||_L \le 1} \big|{\expectation}_{\bs{x} \sim p_l^{(t+\Delta t)}}[f(\bs{x})] - {\expectation}_{\bs{y} \sim p_l^{(t)}} [f(\bs{y})]\big|
	\end{split}
	\end{equation*}
	since $f$ is a $1$-Lipschitz function iff $-f$ is a $1$-Lipschitz function, and 
	\begin{equation*}
	\begin{split}
	\big|{\expectation}_{\bs{x} \sim p_l^{(t+\Delta t)}}[f(\bs{x})] - {\expectation}_{\bs{y} \sim p_l^{(t)}} [f(\bs{y})]\big|
	= &\max\Bigg\{ {\expectation}_{\bs{x} \sim p_l^{(t+\Delta t)}}[f(\bs{x})] - {\expectation}_{\bs{y} \sim p_l^{(t)}} [f(\bs{y})], 
	{\expectation}_{\bs{x} \sim p_l^{(t+\Delta t)}}[-f(\bs{x})] - {\expectation}_{\bs{y} \sim p_l^{(t)}} [-f(\bs{y})]\Bigg\}
	\end{split}
	\end{equation*}
\end{fact}

\vspace{20mm}
\begin{fact}
	\label{fact:2}
	\begin{equation*}
	\begin{split}
	W(p_l^{(t+\Delta t)}, p_l^{(t)}) = & \sup_{||f||_L \le 1} {\expectation}_{\bs{x} \sim p_l^{(t+\Delta t)}}[f(\bs{x})] - {\expectation}_{\bs{y} \sim p_l^{(t)}} [f(\bs{y})] \\
	= & \sup_{f: ||f||_L \le 1, f(\bs{0}) = 0} {\expectation}_{\bs{x} \sim p_l^{(t+\Delta t)}}[f(\bs{x})] - {\expectation}_{\bs{y} \sim p_l^{(t)}} [f(\bs{y})]
	\end{split}
	\end{equation*}
	since for any $1$-Lipschitz function $f$, 
	\begin{equation*}
	\begin{split}
	{\expectation}_{\bs{x} \sim p_l^{(t+\Delta)}}\big[f(\bs{x})\big] - {\expectation}_{\bs{y} \sim p_l^{(t)}}\big[f(\bs{y})\big] 
	=& {\expectation}_{\bs{x} \sim p_l^{(t+\Delta)}}\big[f(\bs{x})\big] - f(\bs{0}) - {\expectation}_{\bs{y} \sim p_l^{(t)}}\big[f(\bs{y})\big] + f(\bs{0}) \\
	=& {\expectation}_{\bs{x} \sim p_l^{(t+\Delta)}}\big[f(\bs{x}) - f(\bs{0})\big] - {\expectation}_{\bs{y} \sim p_l^{(t)}}\big[f(\bs{y}) - f(\bs{0})\big] \\
	=& {\expectation}_{\bs{x} \sim p_l^{(t+\Delta)}}\big[\tilde{f}(\bs{x})\big] - {\expectation}_{\bs{y} \sim p_l^{(t)}}\big[\tilde{f}(\bs{y})\big]
	\end{split}
	\end{equation*}
	where $\tilde{f}(\bs{x}) = f(\bs{x}) - f(\bs{0})$ is still a $1$-Lipschitz function, and satisfies $\tilde{f}(\bs{0}) = 0$.
\end{fact}

\newpage
\subsection{Proofs}
\begin{theorem}
	Suppose that $|\mu_i^{(t)}| < \infty, |\mu_i^{(t + \Delta t)}| < \infty, 1 \le i \le d$. Then, 
	\begin{equation*}
	\begin{split}
	W(p_l^{(t+\Delta t)}, p_l^{(t)}) \le & \sum_{i=1}^d (\sigma_{i}^{(t+\Delta t)})^2 + \sum_{i=1}^d (\sigma_{i}^{(t)})^2 
	+ \Big(\sum_{i=1}^d (\mu_i^{(t+\Delta t)} - \mu_i^{(t)})^2\Big)^{\frac{1}{2}} + 2
	\end{split}
	\end{equation*}
\end{theorem}

\begin{proof}
	Denote by ${\indicator}_A(x)$ an indicator function defined by ${\indicator}_A(x) = 1$ if $x \in A$, otherwise ${\indicator}_A(x) = 0$. Then, according to \tbf{Fact 1}, 
	\begin{equation*}
	\begin{split}
	W(p_l^{(t+\Delta t)}, p_l^{(t)}) 
	=& \sup_{||f||_L \le 1} \big|{\expectation}_{\bs{x} \sim p_l^{(t+\Delta t)}}[f(\bs{x})] - {\expectation}_{\bs{y} \sim p_l^{(t)}} [f(\bs{y})]\big| \\
	\le& \sup_{||f||_L \le 1} \big|{\expectation}_{\bs{x} \sim p_l^{(t+\Delta t)}}[f(\bs{x}) - f(\bs{\mu}^{(t+\Delta t)})]  
	- {\expectation}_{\bs{y} \sim p_l^{(t)}} [f(\bs{y})- f(\bs{\mu}^{(t)})]\big| + \big| f(\bs{\mu}^{(t+\Delta t)})- f(\bs{\mu}^{(t)}) \big| \\
	\le& \sup_{||f||_L \le 1} {\expectation}_{\bs{x} \sim p_l^{(t+\Delta t)}}\big[\big|f(\bs{x}) - f(\bs{\mu}^{(t+\Delta t)})\big|\big] 
	+ {\expectation}_{\bs{y} \sim p_l^{(t)}}\big[\big|f(\bs{y}) - f(\bs{\mu}^{(t)})\big|\big]
	+ \big| f(\bs{\mu}^{(t+\Delta t)})- f(\bs{\mu}^{(t)}) \big| \\
	\le& \sup_{||f||_L \le 1} {\expectation}_{\bs{x} \sim p_l^{(t+\Delta t)}}\big[\big|\big|\bs{x} - \bs{\mu}^{(t+\Delta t)}\big|\big|_2\big] 
	+ {\expectation}_{\bs{y} \sim p_l^{(t)}}\big[\big|\big|\bs{y} - \bs{\mu}^{(t)}\big|\big|_2\big] + \big|\big|\bs{\mu}^{(t+\Delta t)}- \bs{\mu}^{(t)}\big|\big|_2 \\
	=&  {\expectation}_{\bs{x} \sim p_l^{(t+\Delta t)}}\big[{\indicator}_{||\bs{x} - \bs{\mu}^{(t+\Delta t)}||_2 \le 1}(\bs{x})\big|\big|\bs{x} - \bs{\mu}^{(t+\Delta t)}\big|\big|_2\big]
	+ {\expectation}_{\bs{x} \sim p_l^{(t+\Delta t)}}\big[{\indicator}_{||\bs{x} - \bs{\mu}^{(t+\Delta t)}||_2 > 1}(\bs{x})\big|\big|\bs{x} - \bs{\mu}^{(t+\Delta t)}\big|\big|_2\big]\\
	& + {\expectation}_{\bs{y} \sim p_l^{(t)}}\big[{\indicator}_{||\bs{y} - \bs{\mu}^{(t)}||_2 \le 1}(\bs{y})\big|\big|\bs{y} - \bs{\mu}^{(t)}\big|\big|_2\big] 
	+ {\expectation}_{\bs{y} \sim p_l^{(t)}}\big[{\indicator}_{||\bs{y} - \bs{\mu}^{(t)}||_2 > 1}(\bs{y})\big|\big|\bs{y} - \bs{\mu}^{(t)}\big|\big|_2\big] 
	+ \big|\big|\bs{\mu}^{(t+\Delta t)}- \bs{\mu}^{(t)}\big|\big|_2\\
	\le &  {\expectation}_{\bs{x} \sim p_l^{(t+\Delta t)}}\big[{\indicator}_{||\bs{x} - \bs{\mu}^{(t+\Delta t)}||_2 \le 1}(\bs{x})\cdot 1\big] 
	+ {\expectation}_{\bs{x} \sim p_l^{(t+\Delta t)}}\big[{\indicator}_{||\bs{x} - \bs{\mu}^{(t+\Delta t)}||_2 > 1}(\bs{x})\big|\big|\bs{x} - \bs{\mu}^{(t+\Delta t)}\big|\big|_2^2\big] \\
	& + {\expectation}_{\bs{y} \sim p_l^{(t)}}\big[{\indicator}_{||\bs{y} - \bs{\mu}^{(t)}||_2 \le 1}(\bs{y})\cdot 1\big] 
	+ {\expectation}_{\bs{y} \sim p_l^{(t)}}\big[{\indicator}_{||\bs{y} - \bs{\mu}^{(t)}||_2 > 1}(\bs{y})\big|\big|\bs{y} - \bs{\mu}^{(t)}\big|\big|_2^2\big] 
	+ \big|\big|\bs{\mu}^{(t+\Delta t)}- \bs{\mu}^{(t)}\big|\big|_2\\
	\le &  1 + {\expectation}_{\bs{x} \sim p_l^{(t+\Delta t)}}\bigg[\sum_{i=1}^d (x_i - \mu_i^{(t + \Delta)})^2 \bigg] + 1 
	+ {\expectation}_{\bs{y} \sim p_l^{(t)}}\bigg[\sum_{i=1}^d (y_i - \mu_i^{(t)})^2 \bigg] + \big|\big|\bs{\mu}^{(t+\Delta t)}- \bs{\mu}^{(t)}\big|\big|_2 \\
	=& \sum_{i=1}^d (\sigma_{i}^{(t+\Delta)})^2 + \sum_{i=1}^d (\sigma_{i}^{(t)})^2 + \Big(\sum_{i=1}^d (\mu_i^{(t+\Delta t)} - \mu_i^{(t)})^2\Big)^{\frac{1}{2}} + 2
	\end{split}
	\end{equation*}
\end{proof}

\newpage
\begin{theorem}
	Suppose that $C > 0$ is a real number, and $p \ge 2$ is an integer. Then, 
	\begin{equation*}
	\begin{split}
	W(p_l^{(t + \Delta t)}, p_l^{(t)}) =& \sup_{||f||_L \le 1} {\expectation}_{\bs{x} \sim p_l^{(t + \Delta t)}}[f(\bs{y})] - {\expectation}_{\bs{y} \sim p_l^{(t)}}[f(\bs{y})] \\
	\ge & \big|{\expectation}_{\bs{x} \sim p_l^{(t + \Delta t)}}[f_{p, C}(\bs{x})] - {\expectation}_{\bs{y} \sim p_l^{(t)}}[f_{p, C}(\bs{y})]\big| \\
	\end{split}
	\end{equation*}
	where $f_{p, C}$ is the $1$-Lipschitz function defined as 
	\begin{equation*}
	\begin{split}
	f_{p, C}(\bs{x}) = \dfrac{1}{pC^{p-1}d^{\frac{1}{2}}}\bigg(\sum_{|x_i| \le C} x_i^p + \sum_{x_i < -C} (-C)^p + \sum_{x_i > C} C^p\bigg)
	\end{split}
	\end{equation*}
\end{theorem}

\begin{proof}
	According to \tbf{Fact 1}, it's obvious that the inequality holds if $f_{p, C}$ is a $1$-Lipschitz function. Thus, only the proof of Lipschitz continuity of $f_{p, C}$ is required. For convenience, let $f_{p, C} = f_b / (pC^{p-1}d^{\frac{1}{2}})$, where
	\begin{equation*}
	f_b(\bs{x}) = \sum_{|x_i| \le C} x_i^p + \sum_{x_i < -C} (-C)^p + \sum_{x_i > C} C^p
	\end{equation*}
	Then, we'll prove the Lipschitz continuity of $f_b$. In fact, for any $\bs{v}, \bs{w} \in \realnumber^d$, 
	\begin{equation*}
	\begin{split}
	|f_b(\bs{v}) - f_b(\bs{w})| 
	=& \bigg|\sum_{|v_i| \le C} v_i^p + \sum_{v_i < -C} (-C)^p + \sum_{v_i > C} C^p 
	- \sum_{|w_i| \le C} w_i^p - \sum_{w_i < -C} (-C)^p - \sum_{w_i > C} C^p\bigg|\\
	=& \bigg|\sum_{i:|v_i| \le C, |w_i| \le C}(v_i^p - w_i^p) + \sum_{i: |v_i|\le C, w_i < -C}[v_i^p - (-C)^p] 
	+ \sum_{i: |v_i|\le C, w_i > C}(v_i^p - C^p) \\
	&+ \sum_{i: v_i < -C, |w_i| \le C} [(-C)^p - w_i^p] + \sum_{i: v_i < -C, w_i < -C}[(-C)^p - (-C)^p]
	+ \sum_{i: v_i < -C, w_i > C}[(-C)^p - C^p] \\
	&+ \sum_{i: v_i > C, |w_i| \le C} (C^p - w_i^p) 
	+ \sum_{i: v_i > C, w_i < -C}[C^p - (-C)^p] + \sum_{i: v_i > C, w_i > C}(C^p - C^p) \bigg| \\
	=& \bigg|\sum_{i:|v_i| \le C, |w_i| \le C} (v_i - w_i)\sum_{q = 0}^{p-1} v_i^{p-1-q}w_i^q 
	+ \sum_{i: |v_i| \le C, w_i < -C} [v_i - (-C)] \sum_{q = 0}^{p-1} v_i^{p-1-q}(-C)^q \\
	& + \sum_{i: |v_i| \le C, w_i > C}(v_i - C) \sum_{q = 0}^{p-1} v_i^{p-1-q}C^q 
	+ \sum_{i: v_i < -C, |w_i| \le C}(-C - w_i) \sum_{q = 0}^{p-1} (-C)^{p-1-q}w_i^q \\
	& + \sum_{i: v_i < -C, w_i > C}(-C - C)\sum_{q = 0}^{p-1} (-C)^{p-1-q}C^q 
	+ \sum_{i: v_i > C, |w_i| \le C}(C - w_i)\sum_{q=0}^{p-1} C^{p-1-q}w_i^q\\
	& + \sum_{i: v_i > C, w_i < -C}[C - (-C)]\sum_{q = 0}^{p-1} C^{p-1-q} (-C)^{q} \bigg| \\
	\le& pC^{p-1}\bigg(\sum_{i:|v_i| \le C, |w_i| \le C} |v_i - w_i|
	+ \sum_{i: |v_i| \le C, w_i < -C} |v_i - (-C)| + \sum_{i: |v_i| \le C, w_i > C}|v_i - C|  \\
	& + \sum_{i: v_i < -C, |w_i| \le C}|-C - w_i| + \sum_{i: v_i < -C, w_i > C}|-C - C|
	+ \sum_{i: v_i > C, |w_i| \le C}|C - w_i| \\
	&+ \sum_{i: v_i > C, w_i < -C}|C - (-C)| \bigg)\\
	\end{split}
	\end{equation*}
	\begin{equation*}
	\begin{split}
	\le& pC^{p-1}\bigg(\sum_{i:|v_i| \le C, |w_i| \le C} |v_i - w_i| 
	+ \sum_{i: |v_i| \le C, w_i < -C} |v_i - w_i| + \sum_{i: |v_i| \le C, w_i > C}|v_i - w_i| \\
	& + \sum_{i: v_i < -C, |w_i| \le C}|v_i - w_i| + \sum_{i: v_i < -C, w_i > C}|v_i - w_i| 
	+ \sum_{i: v_i > C, |w_i| \le C}|v_i - w_i| + \sum_{i: v_i > C, w_i < -C}|v_i - w_i| \bigg)\\
	\le& pC^{p-1}\sum_{i=1}^d |v_i - w_i| \\
	=& pC^{p-1} ||\bs{v} - \bs{w}||_1 
	\end{split}
	\end{equation*}
	Note that H\"older's inequality implies that
	\begin{equation*}
	\begin{split}
	||\bs{x}||_1 = \sum_{i = 1}^d |x_i| \le \Bigg(\sum_{j=1}^d |x_j|^2 \Bigg)^{1 / 2} \Bigg(\sum_{k=1}^d 1 \Bigg)^{1 / 2} = d^{\frac{1}{2}} ||\bs{x}||_2
	\end{split}
	\end{equation*}
	for any $\bs{x} \in \realnumber^d$. Thus, 
	\begin{equation*}
	|f_b(\bs{v}) - f_b(\bs{w})| \le pC^{p-1}d^{\frac{1}{2}} ||\bs{v} - \bs{w}||_2
	\end{equation*}
	Therefore, 
	\begin{equation*}
	|f_{p, C}(\bs{v}) - f_{p, C}(\bs{w})| = \dfrac{1}{pC^{p-1}d^{\frac{1}{2}}} |f_b(\bs{v}) - f_b(\bs{w})| \le ||\bs{v} - \bs{w}||_2
	\end{equation*} 
	implying $f_{p, C}$ is a $1$-Lipschitz function.
\end{proof}

\noindent
\tbf{Example of the Unbounded EM Distance}: Given $C^\prime > 0$, consider the two distributions as follows:
\begin{equation*}
p_l^{(t + \Delta t)}(\bs{x}) = \left\{
\begin{array}{ll}
\dfrac{1}{\int_{[C^\prime/2, C^\prime]^d} 1 d\bs{x}} &, \bs{x} \in [C^\prime/2, C^\prime]^d \\
0 &, other
\end{array}
\right.
\end{equation*}
and
\begin{equation*}
p_l^{(t)}(\bs{y}) = \left\{
\begin{array}{ll}
\dfrac{1}{\int_{[0, C^\prime/4]^d} 1 d\bs{y}} &, \bs{y} \in [0, C^\prime/4]^d \\
0 &, other
\end{array}
\right.
\end{equation*}
Then, the supports of the distributions are subsets of $[0, C^\prime]^d$ and for any $i \in \{1, 2, \ldots, d\}$, 
\begin{equation*}
\begin{split}
{\expectation}_{\bs{x} \sim p_l^{(t + \Delta t)}}[x_i^p] > \dfrac{C^{\prime p}}{2^p} 
, {\expectation}_{\bs{y} \sim p_l^{(t)}}[y_i^p] < \dfrac{C^{\prime p}}{4^p} 
\end{split}
\end{equation*}
Thus, according to \tbf{Theorem 2}, the lower bound on the EM distance between $p_l^{(t + \Delta)}$ and $p_l^{(t)}$ is
\begin{equation*}
\begin{split}
W(p_l^{(t + \Delta t)}, p_l^{(t)}) \ge & \dfrac{1}{pC^{\prime p-1}d^{\frac{1}{2}}} \bigg|{\expectation}_{\bs{x} \sim p_l^{(t + \Delta t)}}\bigg[\sum_{i=1}^d x_i^p\bigg] - {\expectation}_{\bs{y} \sim p_l^{(t)}}\bigg[\sum_{i=1}^d y_i^p\bigg]\bigg| \\
= & \dfrac{1}{pC^{\prime p-1}d^{\frac{1}{2}}} \bigg|\sum_{i=1}^d{\expectation}_{\bs{x} \sim p_l^{(t + \Delta t)}}[x_i^p] - {\expectation}_{\bs{y} \sim p_l^{(t)}}[y_i^p]\bigg| \\
> & \dfrac{(2^{-p} - 4^{-p})d^{\frac{1}{2}}}{p} C^\prime
\end{split}
\end{equation*}
Note that $p$ and $d$ are fixed, and then the lower bound is dominated by $C^\prime$. Thus, the distance is unbounded and would go to infinity as $C^\prime \rightarrow \infty$. Intuitively, the samples from $p_l^{(t + \Delta t)}$ can be regarded as the scaled and shifted samples from $p_l^{(t)}$. The unstable scale and center of the distributions lead to the unbounded lower bound, implying the unbounded distance and upper bound. The distributions of deep layers' outputs might perform in the same way if they are not normalized. In contrast, the upper bound on the distance for the outputs processed by BN is relatively stable (though it depends on the moments with noise), and can constrain the distance to a reasonable range in some cases discussed before.

\newpage
\stepcounter{subtheoremcounter}
\stepcounter{subtheoremcounter}
\stepcounter{subtheoremcounter}
\begin{subtheorem}
	Suppose that for $\bs{x} \sim p_l^{(t)}$, $g(\bs{x}) \sim p_U^{(t)}$. Then, 
	\begin{equation*}
	\begin{split}
	W(p_U^{(t + \Delta t)}, p_U^{(t)}) \le 2
	\end{split}
	\end{equation*}
\end{subtheorem}
\begin{proof}
	Note that for any $1$-Lipschitz function $f$ such that $f(\bs{0}) = 0$, 
	\begin{equation*}
	\begin{split}
	\big|f(g(\bs{x}))\big| = \big|f(g(\bs{x})) - f(\bs{0})\big| \le \big|\big|g(\bs{x}) - \bs{0}\big|\big|_2 = 1, \forall \bs{x}
	\end{split}
	\end{equation*}
	which yields
	\begin{equation*}
	-1 \le f(g(\bs{x}))  \le 1, \forall \bs{x}
	\end{equation*}
	Therefore, according to \tbf{Fact 2},
	\begin{equation*}
	\begin{split}
	W(p_U^{(t + \Delta t)}, p_U^{(t)}) 
	= &\sup_{||f||_L \le 1} {\expectation}_{\bs{x} \sim p_l^{(t+\Delta)}}\big[f(g(\bs{x}))\big] - {\expectation}_{\bs{y} \sim p_l^{(t)}}\big[f(g(\bs{y}))\big] \\
	= &\sup_{f: ||f||_L \le 1, f(\bs{0}) = 0} {\expectation}_{\bs{x} \sim p_l^{(t+\Delta)}}\big[f(g(\bs{x}))\big] - {\expectation}_{\bs{y} \sim p_l^{(t)}}\big[f(g(\bs{y}))\big] \\
	\le &\sup_{f: ||f||_L \le 1, f(\bs{0}) = 0} {\expectation}_{\bs{x} \sim p^{(t+\Delta)}}[1] - {\expectation}_{\bs{y} \sim p^{(t)}}[-1] \\
	= &2
	\end{split}
	\end{equation*}
\end{proof}

\begin{subtheorem}
	Suppose that for $\alpha \in [0, 1]$ and $\bs{x} \sim p_l^{(t)}$, $g(\bs{x}; \alpha) \sim p_U^{(t)}$. Then, 
	\begin{equation*}
	\begin{split}
	W(p_U^{(t + \Delta t)}, p_U^{(t)}) 
	\le & {\indicator}_{\alpha = 0}(\alpha) \cdot ({\expectation}_{\bs{x} \sim p_l^{(t+\Delta)}}[||\bs{x}||_2]
	+ {\expectation}_{\bs{y} \sim p_l^{(t)}}[||\bs{y}||_2]) + {\indicator}_{\alpha > 0}(\alpha) \cdot \dfrac{2}{\alpha}
	\end{split}
	\end{equation*}
\end{subtheorem}
\begin{proof}
	Note that given a $1$-Lipschitz function $f$ that satisfies $f(\bs{0}) = 0$, for any $\bs{x}$ and $\alpha \in [0, 1]$, 
	\begin{equation*}
	\begin{split}
	\big|f(g(\bs{x}; \alpha))\big| 
	= & \big|f(g(\bs{x}; \alpha)) - f(\bs{0})\big|\\
	\le& \big|\big|g(\bs{x}; \alpha) - \bs{0}\big|\big|_2 \\
	=& {\indicator}_{||\bs{x}||_2 > 0}(\bs{x}) \cdot \dfrac{||\bs{x}||_2}{\alpha ||\bs{x}||_2 + (1 - \alpha) \times 1} 
	+ {\indicator}_{||\bs{x}||_2 = 0, \alpha = 1}(\bs{x}, \alpha) \cdot 1 
	+ {\indicator}_{||\bs{x}||_2 = 0, \alpha < 1}(\bs{x}, \alpha) \cdot 0 \\
	=& {\indicator}_{||\bs{x}||_2 > 0, \alpha = 0}(\bs{x}, \alpha) \cdot \dfrac{||\bs{x}||_2}{\alpha ||\bs{x}||_2 + (1 - \alpha) \times 1} 
	+ {\indicator}_{||\bs{x}||_2 > 0, \alpha > 0}(\bs{x}, \alpha) \cdot \dfrac{1}{\alpha + (1 - \alpha) / ||\bs{x}||_2}
	+ {\indicator}_{||\bs{x}||_2 = 0, \alpha = 1}(\bs{x}, \alpha) \cdot 1 \\
	\le& {\indicator}_{||\bs{x}||_2 > 0, \alpha = 0}(\bs{x}, \alpha) \cdot ||\bs{x}||_2
	+ {\indicator}_{||\bs{x}||_2 > 0, \alpha > 0}(\bs{x}, \alpha) \cdot \dfrac{1}{\alpha} 
	+ {\indicator}_{||\bs{x}||_2 = 0, \alpha > 0}(\bs{x}, \alpha) \cdot 1 \\
	\le& {\indicator}_{\alpha = 0}(\alpha) \cdot ||\bs{x}||_2 + {\indicator}_{\alpha > 0}(\alpha) \cdot \dfrac{1}{\alpha}
	\end{split}
	\end{equation*}
	which yields
	\begin{equation*}
	\begin{split}
	-{\indicator}_{\alpha = 0}(\alpha) \cdot ||\bs{x}||_2 - {\indicator}_{\alpha > 0}(\alpha) \cdot \dfrac{1}{\alpha} \le
	f(g(\bs{x}; \alpha)) \le {\indicator}_{\alpha = 0}(\alpha) \cdot ||\bs{x}||_2 + {\indicator}_{\alpha > 0}(\alpha) \cdot \dfrac{1}{\alpha} 
	\end{split}
	\end{equation*}
	Therefore, according to \tbf{Fact 2},
	\begin{equation*}
	\begin{split}
	W(p_U^{(t + \Delta t)}, p_U^{(t)}) = &\sup_{||f||_L \le 1} {\expectation}_{\bs{x} \sim p_l^{(t+\Delta)}}\big[f(g(\bs{x};\alpha))\big]
	- {\expectation}_{\bs{y} \sim p_l^{(t)}}\big[f(g(\bs{y};\alpha))\big] \\
	= &\sup_{f: ||f||_L \le 1, f(\bs{0}) = 0} {\expectation}_{\bs{x} \sim p_l^{(t+\Delta)}}\big[f(g(\bs{x};\alpha))\big] 
	- {\expectation}_{\bs{y} \sim p_l^{(t)}}\big[f(g(\bs{y};\alpha))\big] \\
	\le & {\indicator}_{\alpha = 0}(\alpha) \cdot ({\expectation}_{\bs{x} \sim p_l^{(t+\Delta)}}[||\bs{x}||_2] 
	+ {\expectation}_{\bs{y} \sim p_l^{(t)}}[||\bs{y}||_2]) + {\indicator}_{\alpha > 0}(\alpha) \cdot \dfrac{2}{\alpha}
	\end{split}
	\end{equation*}
\end{proof}

\newpage
\begin{subtheorem}
	Suppose that for $\bs{\alpha} = (\alpha_1, \alpha_2, \ldots, \alpha_d)$, where $\alpha_i \in [0, 1], 1 \le i \le d$, and $\bs{x} \sim p_l^{(t)}$, $g(\bs{x}; \bs{\alpha}) \sim p_U^{(t)}$. Then, 
	\begin{equation*}
	\begin{split}
	W(p_U^{(t + \Delta t)}, p_U^{(t)}) \le & {\indicator}_{\min_j \alpha_j > 0}(\bs{\alpha}) \cdot \dfrac{2}{\min_j \alpha_j} 
	+ {\indicator}_{\min_j \alpha_j = 0}(\bs{\alpha}) \cdot ({\expectation}_{\bs{x} \sim p_l^{(t+\Delta)}}[||\bs{x}||_2] 
	+ {\expectation}_{\bs{y} \sim p_l^{(t)}}[||\bs{y}||_2] + 2)
	\end{split}
	\end{equation*}
\end{subtheorem}
\indent
\begin{proof}
	Given a $1$-Lipschitz function $f$ satisfying $f(\bs{0}) = 0$, for any $\bs{x}$ and $\bs{\alpha} = (\alpha_1, \alpha_2, \ldots, \alpha_d) \in [0, 1]^d$, 
	\begin{equation*}
	\begin{split}
	\big|f(g(\bs{x}; \bs{\alpha}))\big| \le& \big|\big|g(\bs{x}; \bs{\alpha})\big|\big|_2 = {\indicator}_{||\bs{x}||_2 = 0}(\bs{x}) \cdot 0 + {\indicator}_{||\bs{x}||_2 > 0}(\bs{y}) \cdot \Bigg[
	\sum_{i = 1}^{d} \bigg(\dfrac{x_i}{\alpha_i ||\bs{x}||_2 + (1 - \alpha_i) \cdot 1} \bigg)^2 \Bigg]^{\frac{1}{2}} \\
	\le& {\indicator}_{||\bs{x}||_2 > 0}(\bs{x}) \cdot \Bigg[\sum_{i = 1}^{d} \bigg( \dfrac{x_i}{\min_{j} \alpha_j ||\bs{x}||_2 + (1 - \alpha_j)} \bigg)^2 \Bigg]^{\frac{1}{2}} \\
	\le& {\indicator}_{||\bs{x}||_2 > 0}(\bs{x}) \cdot \dfrac{||\bs{x}||_2}{\min_{j} \alpha_j ||\bs{x}||_2 + (1 - \alpha_j)} \\
	=& {\indicator}_{||\bs{x}||_2 > 0}(\bs{x}) \cdot \bigg({\indicator}_{\min_j \alpha_j > 0}(\bs{\alpha}) \cdot \dfrac{1}{\min_{j} \alpha_j + (1 - \alpha_j) / ||\bs{x}||_2} 
	+ {\indicator}_{\min_j \alpha_j = 0}(\bs{\alpha}) \cdot \dfrac{||\bs{x}||_2}{\min_{j} \alpha_j ||\bs{x}||_2 + (1 - \alpha_j)}\bigg)\\
	\le& {\indicator}_{||\bs{x}||_2 > 0}(\bs{x}) \cdot \bigg({\indicator}_{\min_j \alpha_j > 0}(\bs{\alpha}) \cdot \dfrac{1}{\min_j \alpha_j} 
	+ {\indicator}_{\min_j \alpha_j = 0}(\bs{\alpha}) \cdot \max\big\{||\bs{x}||_2, 1\big\}\bigg) \\
	\le& {\indicator}_{\min_j \alpha_j > 0}(\bs{\alpha}) \cdot \dfrac{1}{\min_j \alpha_j} + {\indicator}_{\min_j \alpha_j = 0}(\bs{\alpha}) \cdot \max\big\{||\bs{x}||_2, 1\big\}
	\end{split}
	\end{equation*}
	which yields
	\begin{equation*}
	\begin{split}
	f(g(\bs{x}; \bs{\alpha})) \ge& -{\indicator}_{\min_j \alpha_j > 0}(\bs{\alpha}) \cdot \dfrac{1}{\min_j \alpha_j} 
	- {\indicator}_{\min_j \alpha_j = 0}(\bs{\alpha}) \cdot \max\big\{||\bs{x}||_2, 1\big\} \\
	f(g(\bs{x}; \bs{\alpha})) \le& {\indicator}_{\min_j \alpha_j > 0}(\bs{\alpha}) \cdot \dfrac{1}{\min_j \alpha_j} 
	+ {\indicator}_{\min_j \alpha_j = 0}(\bs{\alpha}) \cdot \max\big\{||\bs{x}||_2, 1\big\}
	\end{split}
	\end{equation*}
	Therefore, according to \tbf{Fact 2},
	\begin{equation*}
	\begin{split}
	W(p_U^{(t + \Delta t)}, p_U^{(t)}) 
	= &\sup_{||f||_L \le 1} {\expectation}_{\bs{x} \sim p_l^{(t+\Delta)}}\big[f(g(\bs{x};\bs{\alpha}))\big] - {\expectation}_{\bs{y} \sim p_l^{(t)}}\big[f(g(\bs{y};\bs{\alpha}))\big] \\
	= &\sup_{f: ||f||_L \le 1, f(\bs{0}) = 0} {\expectation}_{\bs{x} \sim p_l^{(t+\Delta)}}\big[f(g(\bs{x};\bs{\alpha}))\big] - {\expectation}_{\bs{y} \sim p_l^{(t)}}\big[f(g(\bs{y};\bs{\alpha}))\big] \\
	\le & {\indicator}_{\min_j \alpha_j > 0}(\bs{\alpha}) \cdot \dfrac{2}{\min_j \alpha_j} + {\indicator}_{\min_j \alpha_j = 0}(\bs{\alpha}) 
	\cdot \bigg({\expectation}_{\bs{x} \sim p_l^{(t+\Delta)}}\Big[\max\big\{||\bs{x}||_2, 1\big\}\Big] + {\expectation}_{\bs{y} \sim p_l^{(t)}}\Big[\max\big\{||\bs{y}||_2, 1\big\}\Big]\bigg) \\
	\le & {\indicator}_{\min_j \alpha_j > 0}(\bs{\alpha}) \cdot \dfrac{2}{\min_j \alpha_j} + {\indicator}_{\min_j \alpha_j = 0}(\bs{\alpha}) 
	\cdot \bigg({\expectation}_{\bs{x} \sim p_l^{(t+\Delta)}}[||\bs{x}||_2] + {\expectation}_{\bs{x} \sim p_l^{(t+\Delta)}}[1] + {\expectation}_{\bs{y} \sim p_l^{(t)}}[||\bs{y}||_2] 
	+ {\expectation}_{\bs{y} \sim p_l^{(t)}}[1]\bigg) \\
	\le & {\indicator}_{\min_j \alpha_j > 0}(\bs{\alpha}) \cdot \dfrac{2}{\min_j \alpha_j} + {\indicator}_{\min_j \alpha_j = 0}(\bs{\alpha}) 
	\cdot ({\expectation}_{\bs{x} \sim p_l^{(t+\Delta)}}[||\bs{x}||_2] + {\expectation}_{\bs{y} \sim p_l^{(t)}}[||\bs{y}||_2] + 2)
	\end{split}
	\end{equation*}
\end{proof}

\newpage
\section{Earth Mover Distance Estimation}
In the theoretical analysis, the EM distance has been used to measure ICS. However, whether there is a correlation between the performance of networks and ICS in practice is unclear. Thus, an algorithm is proposed to estimate the EM distance in network training, and the results of the following experiments yield evidence that ICS is related to the networks' performance.

\subsection{Algorithm}
By definition, the accurate EM distance is obtained by optimizing a function $f$ over the 1-Lipschitz function space:
\begin{equation*}
W(p_l^{(t+\Delta t)}, p_l^{(t)}) = \sup_{||f||_L \le 1} {\expectation}_{\bs{x}\sim p_l^{(t+\Delta t)}}[f(\bs{x})] - {\expectation}_{\bs{y} \sim p_l^{(t)}}[f(\bs{y})]
\end{equation*}
Thus, the proposed algorithm aims to get the optimal 1-Lipschitz function $f^*$ and use $f^*$ to compute the difference between the expectations. In WGAN~\cite{arjovsky2017wasserstein}, the 1-Lipschitz function $f$ is parameterized by a neural network $f_w$ with a scalar output in the algorithm. Then, the parameters of $f_w$ are bounded to satisfy the constraint $||f_w||_L \le 1$. Furthermore, the constraint can be relaxed such that $w$ is required to satisfy $||f_w||_L \le K$ for some $K > 0$, and then the EM distance is $K \cdot W(p_l^{(t + \Delta t)}, p_l^{(t)})$. For convenience, the parameter space is denoted by $\mathcal{W} = \big\{w\big|||f_w||_L \le K\big\}$. Then, a suboptimal $K$-Lipschitz function $f_w^*$ is obtained by training $f_w$ on the dataset consisting of the real and generated images. The EM distance is computed by $f_w^*$'s outputs with the real and generated images as the inputs. 

In a convolutional neural network $f_{1:L}$ with $L$ layers, the EM distance for measuring ICS is estimated in the same way. For the $l$th ($l < L$) convolutional layer of $f_{1:L}$, the multi-channel outputs can be treated as images, though there are more than $3$ channels. Thus, analogously, the proposed algorithm trains $f_w$ to maximize the same objective function of WGANs, with the training samples generated by the local networks $f_{1:l}^{(t)}$ and $f_{1:l}^{(t + \Delta t)}$ formed by the previous $l$ layers of $f_{1:L}$ at the $t$th and $t + \Delta t$th iterations, respectively. Then, the algorithm estimates the EM distance by $f_w^*$ similarly.

To formulate the process, the approximated EM distance, which replaces the 1-Lipschitz function space with $\mathcal{W}$, is defined as
\begin{equation*}
\begin{split}
\widetilde{W}(p_l^{(t + \Delta t)}, p_l^{(t)}) = \sup_{w \in \mathcal{W}} E_{\bs{x}\sim p_l^{(t + \Delta t)}}[f_w(\bs{x})] - E_{\bs{y} \sim p_l^{(t)}}[f_w(\bs{y})]
\end{split}
\label{eq:approximatedics}
\end{equation*}
Then, in practice, $\widetilde{W}(p_l^{(t + \Delta t)}, p_l^{(t)})$ is further approximated by the empirical estimation:
\begin{equation*}
\begin{split}
\widetilde{W}(p_l^{(t + \Delta t)}, p_l^{(t)}) \approx \sup_{w\in \mathcal{W}}& \frac{1}{N}\sum_{i=1}^N f_w(f_{1:l}^{(t + \Delta t)}(\bs{x}_i)) - \frac{1}{N}\sum_{i=1}^N f_w(f_{1:l}^{(t)}(\bs{x}_i))
\end{split}
\end{equation*}
where $N$ is the number of samples and $\bs{x}_i$ is a sample from the dataset for $f_{1:L}$. The algorithm is presented in Algorithm~\ref{alg:emdistance}.

\setcounter{algorithm}{2}
\begin{algorithm}[htb]
	\textbf{Input}: datasets $\mathcal{D}_{train}$ and $\mathcal{D}_{test}$, local networks $f_{1:l}^{(t + \Delta t)}$ and $f_{1:l}^{(t)}$, initialized $f_w$, parameters $T$, $n$ and $c$
	
	\textbf{Output}: estimated EM distance $d$
	
	\begin{algorithmic}[1]
		\FOR{$i \gets 1$ to $T$}
		\STATE{Sample $\{\bs{x}_j\}_{j=1}^{n}$ randomly from $\mathcal{D}_{train}$}
		\STATE{$g_w \gets \nabla_w \Big[\frac{1}{n}\sum_{j=1}^{n}\big(f_w(f_{1:l}^{(t + \Delta t)}(\bs{x}_j)) $$ - f_w(f_{1:l}^{(t)}(\bs{x}_j))\big)\Big]$}
		\STATE{Update $w$ by $g_w$}
		\STATE{$w \gets clip(w, -c, c)$}
		\ENDFOR
		\STATE{$d \gets \frac{1}{|\mathcal{D}_{test}|} \sum_{\bs{x} \in \mathcal{D}_{test}} \big(f_w(f_{1:l}^{(t + \Delta t)}(\bs{x})) - f_w(f_{1:l}^{(t)}(\bs{x}))\big)$}
	\end{algorithmic}
	\caption{EM Distance Estimation Algorithm}
	\label{alg:emdistance}
\end{algorithm}

\subsection{Experiments}
The EM distances for some specific unitization layers of ResNet-110s are estimated. The ResNet-110s, \ie, $f_{1:L}$s, are trained on CIFAR-10 and CIFAR-100 datasets with the data augmentation method~\cite{lee2015deeply}, while the $f_w$s corresponding to the specific layers are trained on the same dataset without data augmentation.

\begin{figure*}[htb]
	\centering
	\begin{subfigure}[t]{0.45\textwidth}
		\centering
		\includegraphics[width=80mm,scale=0.4]{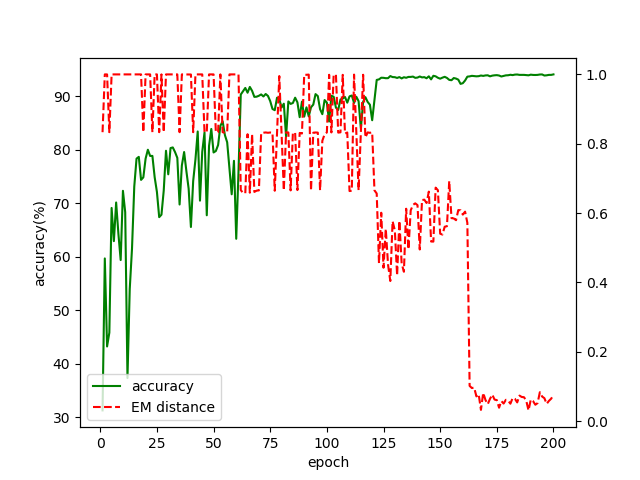}
		\caption{CIFAR-10}
	\end{subfigure}
	\begin{subfigure}[t]{0.45\textwidth}
		\centering
		\includegraphics[width=80mm,scale=0.4]{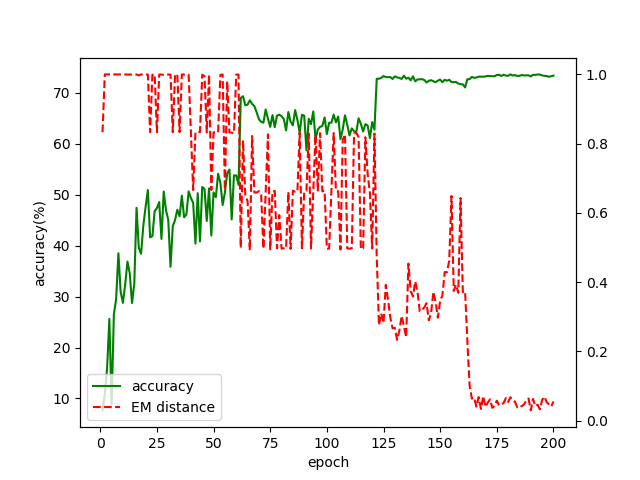}
		\caption{CIFAR-100}
	\end{subfigure}
	\caption{Illustration of test accuracy of ResNet-110s on the CIFAR datasets and the average EM distances. On both the CIFAR-10 (a) and CIFAR-100 (b) datasets, the accuracy increases significantly in about $60$th, $120$th and $160$th epochs, and the average EM distance sharply drops concurrently.}
	\label{fig:emdacc}
\end{figure*}

\tbf{Network Architectures}: Each $f_w$ is parameterized as a vanilla CNN. The inputs of $f_w$ are the outputs of the corresponding layer, followed by four $3 \times 3$ convolutions on the feature maps of sizes $\{64, 64, 128, 128\}$ with strides of $\{2, 1, 2, 1\}$, respectively. The convolutional outputs are activated using ReLU but not normalized/unitized. A fully-connected layer with a sigmoid activation is applied to generate the network's outputs. 

\tbf{Implementation Details}: The experiments of training ResNet-110s on the CIFAR datasets have been described in the paper, and the experiments of estimating the distances are described as follows. The EM distances are computed per epoch. At the end of each epoch, the weights of the current local network $f_{1:l}^{(t + \Delta t)}$ are first saved for the next epoch. Then another local network $f_{1:l}^{(t)}$ is initialized by the weights saved in the previous epoch. $f_w$ is trained by Algorithm~\ref{alg:emdistance} with these two local networks as the algorithm's inputs. For each estimated EM distance, the same hyperparameters are used: the iteration number $T$ is $1.5k$; the mini-batch size $n$ is $128$; the bound $c$ is $0.01$. $f_w$ is initialized by the same weights that are generated using the method~\cite{he2015delving}. Finally, the average EM distance for the deep layers including the $60$th, $69$th, $78$th, $87$th, $96$th and $105$th layers, along with the test accuracy of the ResNet-110s, are reported.

\tbf{Results}: As is shown in Figure~\ref{fig:emdacc}, the accuracy of each experiment significantly increases in about $60$th, $120$th and $160$th epoch as the learning rate decreases. Meanwhile, the average EM distance dramatically drops in these epochs. The results demonstrate that the EM distance is a suitable ICS measure since it drops immediately as the learning rate decreases, in which ICS obviously decreases. Then, based on the effectiveness of the EM distance in measuring ICS, the results further substantiate that ICS is related to the performance of networks. 


\vspace{15mm}
\section{Experiments of the Unitization for Micro-Batches}
The unitization is evaluated in the case of micro-batches to further demonstrate the performance. For comparison, BN and GN~\cite{wu2018group} techniques are also evaluated in this case. Only the CIFAR-10 dataset with the same data augmentation in the previous experiments are used. 

\tbf{Network Architectures}: The unitization, BN and GN are evaluated with the same ResNet-110s in the previous experiments. For GN, there are different ways of group division, according to the experiments in~\cite{wu2018group}. For convenience, denote by $G$ the number of the groups in GN, with the value in $\{1, 2, 4, 8, 16\}$. Furthermore, for a network, there are two types of settings for $G$ in each convolutional layer: fixing the numbers of (1) the groups or (2) the channels per group. Thus, there are $10$ ways of group division in total.

\tbf{Implementation Details}: The ResNet-110s are trained with the same settings as the previous experiments, except for the batch size, denoted by $s$. Unlike~\cite{wu2018group}, we experiment on \tbf{only one GPU}, and report the median accuracy of 3 runs for each experiment of $s = 2$. In fact, for the cases of $s \in \{4, 8, 16\}$, BN still outperforms GN. Thus, only the case of $s = 2$, in which BN degrades, is considered.

\begin{table}[!htb]
	\centering
	\caption{Classification accuracy ($\%$) on the CIFAR-10 dataset}
	\small
	\begin{tabular}{Sc c c c c |Sc c c c c Sc  Sc}
		\hline
		\multicolumn{5}{c|}{Groups ($G$)} & \multicolumn{5}{c}{Channels per group} & \multirow{2}{*}{BatchNorm} & \multirow{2}{*}{Unitization} \\ \cline{1-10}
		1 & 2 & 4 & 8 & 16 & 1 & 2 & 4 & 8 & 16 \\
		\Xhline{2pt}
		10.00 & 10.00 & 10.00 & 71.63 & 73.76 & 79.54 & 73.58 & 70.68 & 72.04 & 10.00 & 71.50 & \tbf{81.88} \\
		\hline
	\end{tabular}
	\label{tb:micro_batch}
\end{table}

\tbf{Results}: As is shown in Table~\ref{tb:micro_batch}, the unitization still outperforms the other techniques for micro-batches. Some networks with GN cannot even converge, where the accuracy (of only $10\%$) degrades, and BN suffers from highly noisy estimations in this case, with the almost lowest accuracy among all the comparable results ($> 10\%$). The results further demonstrate the effectiveness of the unitization in controlling ICS.

\end{document}